%% file: main.tex
\documentclass{article}
\usepackage{iclr2027_conference,times}
\input{preamble}

\newif\ifanonymoussubmission
\anonymoussubmissionfalse
\input{repository_public}

\hypersetup{pdftitle={Exact Minimax One-Bit Unbiased Compression: Heavy-Tail Necessity and Finite-Randomness Approximation},pdfauthor={Tao Jiang, Minbo Gao, Shaowei Cai}}

\title{Exact Minimax One-Bit Unbiased Compression:\\Heavy-Tail Necessity and\\Finite-Randomness Approximation}

\author{
Tao Jiang \quad Minbo Gao \quad Shaowei Cai\\
Key Laboratory of System Software (Chinese Academy of Sciences)\\
State Key Laboratory of Computer Science\\
Institute of Software, Chinese Academy of Sciences\\
School of Computer Science and Technology, University of Chinese Academy of Sciences\\
Beijing, China\\
\texttt{\{jiangt,gaomb,caisw\}@ios.ac.cn}
}

\begin{document}
\raggedbottom
\maketitle
\input{paper_body}
\bibliographystyle{plainnat}
\bibliography{references}
\appendix
\input{appendix}
\end{document}

%% file: preamble.tex
\usepackage{amsmath,amssymb,amsthm,mathtools}
\usepackage{booktabs}
\usepackage{array}
\usepackage{microtype}
\usepackage{url}
\usepackage{xcolor}
\usepackage{enumitem}
\usepackage{graphicx}
\usepackage{float}
\usepackage{placeins}
\usepackage{hyperref}
\hypersetup{colorlinks=true,citecolor=blue,linkcolor=blue,urlcolor=blue}
\input{math_commands}

\newtheorem{theorem}{Theorem}
\newtheorem{proposition}[theorem]{Proposition}
\newtheorem{lemma}[theorem]{Lemma}
\newtheorem{corollary}[theorem]{Corollary}

\newtheorem{remark}[theorem]{Remark}
\theoremstyle{definition}

%% file: math_commands.tex
\newcommand{\R}{\mathbb{R}}
\newcommand{\E}{\mathbb{E}}
\newcommand{\Pp}{\mathbb{P}}
\newcommand{\1}{\mathbf{1}}
\newcommand{\cJ}{\mathcal{J}}
\newcommand{\cC}{\mathcal{C}}
\newcommand{\cQ}{\mathcal{Q}}
\newcommand{\cN}{\mathcal{N}}

\newcommand{\Var}{\operatorname{Var}}
\newcommand{\Risk}{\operatorname{Risk}}
\newcommand{\supp}{\operatorname{supp}}

\newcommand{\clip}{\operatorname{clip}}
\newcommand{\dd}{\,\mathrm{d}}
\newcommand{\eps}{\varepsilon}

\newcommand{\barPhi}{\overline\Phi}

%% file: repository_public.tex
\newcommand{\CodeRepositoryURL}{https://github.com/jiangxioabai/onebit_unbiased_compression}

%% file: paper_body.tex
\begin{abstract}
A pointwise-unbiased one-bit compressor reconstructs every real input in expectation while transmitting one bit.  For a scalar source $P$ with CDF $F$, mean $m$, and
\[
\cJ(P)=\int_{\R}\sqrt{F(r)(1-F(r))}\,\dd r,
\]
we prove that the infimum of the source-averaged reconstruction second moment over all public-coin one-bit codes unbiased on $\R$ is $m^2+\cJ(P)^2$.  For regular full-support sources, a distribution-centered random-threshold code attains this value; a converse over arbitrary randomized binary encoders and an equality analysis characterize every attaining code up to null sets, bit relabeling, and public-seed refinement.  For the Gaussian location family $\cN(\mu,\sigma^2)$ with $|\mu|\le c\sigma$, the equal prior on the endpoint means is least favorable and the minimax value is $\sigma^2\Lambda_c^2$.  Exact Gaussian minimax optimality forces a critical heavy tail: at the endpoint means, absolute moments are finite exactly for $p<3$, and $\Pp(W>t)=\Theta(t^{-3}/\sqrt{\log t})$.  A Cauchy-mixture robustification inflates the second moment by at most $1/(1-\eta)$ while making every positive-order absolute moment finite.  Finite-support public randomness with finite decoder means cannot achieve exact unbiasedness on $\R$, but a bounded-output approximation using exactly $R$ shared random bits has explicit bias and second-moment bounds converging to the minimax constant.  Finally, coordinate allocation communicates exactly $B$ bits per Gaussian-gradient query.  On Kim's continuous quadratic hard family, the expected optimization guarantee matches the lower bound in its dependence on $(\sigma,d,B,\eps)$, and a finite-variance high-probability bound incurs only a logarithmic confidence factor.
\end{abstract}

\section{Introduction}

Low-bit communication is a central abstraction in distributed learning, federated optimization, and quantized inference.  A recent oracle-complexity lower bound of \citet{kim2026bitconstrained} isolates a sharp question: if each stochastic-gradient query communicates exactly $B$ bits and the oracle noise is genuinely Gaussian, can one attain
\begin{equation}
T=\Theta\!\left(\frac{\sigma^2d}{\eps^2}\max\left\{1,\frac dB\right\}\right)
\label{eq:target-rate}
\end{equation}
without assuming that stochastic gradients have bounded dynamic range?  Existing achievability in that work uses a bounded-range unbiased quantizer, whereas the lower bound already applies to unbounded Gaussian observations.

A finite message cannot reproduce an arbitrary Gaussian realization deterministically.  The relevant statistical requirement is weaker and more useful: reconstruct every input \emph{in expectation}, while controlling the reconstruction second moment under the source distribution.  Public randomness can select a threshold independently of the sample, and the transmitted comparison bit can be importance-weighted to form an unbiased reconstruction.  This raises three structural questions.

First, what is the exact optimum over \emph{all} one-bit codes, rather than over a chosen threshold family?  Second, what tail behavior is forced by exact second-moment optimality?  Third, can the ideal continuous public coin be replaced by a finite, bounded-output implementation with quantified loss?

For a source $P$, define
\begin{equation}
\cJ(P):=\int_{\R}\sqrt{F(r)(1-F(r))}\,\dd r.
\label{eq:J-functional}
\end{equation}
The fixed-source optimum is $m^2+\cJ(P)^2$.  For regular full-support sources, the attaining code draws a threshold with density proportional to $\sqrt{F(1-F)}$, transmits whether the input exceeds it, and centers the bit by its source probability.  A sharp rearrangement inequality gives a converse over arbitrary randomized binary encoders, while the equality conditions characterize every attaining code up to null sets, bit relabeling, and public-seed refinement.

For $X\sim\cN(\mu,\sigma^2)$ with $|\mu|\le c\sigma$, the equal prior on the two endpoint means is least favorable and induces the source distribution $P_c^\star$.  The resulting code has second-moment risk increasing with $|\mu|$, yielding an exact minimax theorem; at $c=1/2$, the scalar constant is $\Lambda_{1/2}^2=3.2179870327\ldots$.  Essential uniqueness makes the order-three absolute-moment transition unavoidable for every attaining Gaussian minimax code.

Two complementary relaxations address these costs.  A Cauchy-mixture robustification inflates the worst-case second moment by at most $1/(1-\eta)$ and makes $\E|\widehat X|^p$ finite for every $p>0$.  A finite threshold grid driven by $R$ shared bits instead gives bounded outputs and explicit bias, after a matching impossibility result shows that finite-support public randomness with finite decoder means cannot provide exact full-line unbiasedness.

Allocating $s=\min\{B,d\}$ scalar codes across coordinates yields a Gaussian-gradient scheme that communicates exactly $B$ bits per query.  On Kim's continuous quadratic hard family, its expected guarantee attains \eqref{eq:target-rate}; coordinatewise median-of-means aggregation additionally gives a high-probability bound using only finite variance.

\paragraph{Contributions.}
\begin{enumerate}[leftmargin=1.6em]
\item \textbf{General fixed-source optimum and equality characterization.}  We identify the exact infimum for general Borel sources with finite second moment and finite $\cJ(P)$, recover the bounded uniform-source formula as a special case, and prove essential uniqueness for regular full-support sources (Theorems~\ref{thm:fixed-source} and~\ref{thm:uniqueness}).
\item \textbf{Exact Gaussian minimax theorem and necessary heavy tails.}  We identify a least-favorable prior on the endpoint means, derive an attaining minimax code, and prove the sharp absolute-moment transition at order three (Theorems~\ref{thm:gaussian-minimax} and~\ref{thm:tail-phase}).
\item \textbf{Stable near-optimal and finite-randomness implementations.}  We give a Cauchy-mixture robustification with finite moments of every positive order, an impossibility result for finite-support public randomness with finite decoder means, and an explicit bounded-output approximation with quantitative bias and second-moment guarantees (Corollary~\ref{cor:robustified} and Theorem~\ref{thm:finite-coin}).
\item \textbf{Fixed-bit Gaussian-gradient optimization.}  We obtain ideal and finite-randomness expected guarantees, together with a high-probability bound, on Kim's continuous quadratic hard family without assuming bounded stochastic-gradient realizations (Corollaries~\ref{cor:quadratic}, \ref{cor:quadratic-hp}, and~\ref{cor:finite-quadratic}).
\end{enumerate}

\paragraph{Scope.}
Exact minimax constants are established for scalar one-bit coding.  On Kim's continuous quadratic hard family, coordinate allocation yields the lower-bound-matching dependence on $(\sigma,d,B,\eps)$; determining the optimal constant over arbitrary joint $B$-bit vector encoders remains open.  Appendix~\ref{app:general-sgd} gives a strongly-convex extension under localization of the \emph{gradient mean}.

\section{Related Work and the Uniform-Source Precedent}

\paragraph{One-bit scalar coding.}
\citet{benbasat2021onebit} study biased and unbiased one-bit reconstruction on $[0,1]$, with worst-case error and finite or unrestricted shared randomness.  An informal derivation by \citet{brady2021dithering}, also cited by \citet{benbasat2021onebit}, solves a different but very close special case: uniform $X\in[0,1]$ and average variance.  It obtains threshold density $(8/\pi)\sqrt{r(1-r)}$ and optimum average variance $\pi^2/64-1/12$.  Corollary~\ref{cor:uniform} recovers these formulas exactly.  Relative to this uniform-source calculation, the present work extends the optimum to general Borel sources, proves a converse over arbitrary randomized binary encoders, characterizes the equality cases, solves the Gaussian location minimax problem, and derives optimization consequences.  Recent full-line fixed-length unbiased quantizers also appear as components of vector constructions \citep{feng2026hadamard}; their objective is distortion after rotation, whereas ours is the exact source-dependent one-bit second-moment problem.

\paragraph{One-bit mean estimation.}
Gaussian and log-concave mean estimation from one-bit observations has a substantial literature \citep{kipnis2017adaptive,kipnis2022onebit,caiwei2024gaussian}.  Recent work gives order-optimal adaptive and non-adaptive procedures across broad tail classes \citep{lau2026tail,miao2026universal,hu2026interaction}.  These works optimize multi-sample estimation and include a localization phase for an unknown mean.  We study a complementary per-sample primitive: pointwise-unbiased reconstruction and its exact source-averaged second moment.

\paragraph{Gradient compression and robust aggregation.}
Random sparsification, rotations, and adaptive quantization are standard tools for communication-efficient mean estimation and optimization \citep{suresh2017distributed,mayekartyagi2020limits,mayekartyagi2020ratq,vargaftik2021drive}.  Median-of-means and related estimators recover high-probability mean guarantees under finite variance \citep{devroye2016subgaussian}.  We use a standard coordinatewise median-of-means argument to show that the exact code's infinite third absolute moment does not preclude high-probability optimization.

\begin{table}[H]
\centering
\caption{Closest one-bit scalar results.  ``All codes'' means arbitrary randomized binary encoders and two-level decoders.}
\label{tab:related}
\footnotesize
\begin{tabular}{>{\raggedright\arraybackslash}p{0.20\textwidth}>{\raggedright\arraybackslash}p{0.17\textwidth}>{\raggedright\arraybackslash}p{0.18\textwidth}>{\raggedright\arraybackslash}p{0.28\textwidth}}
\toprule
Work & Input/source & Criterion & Main distinction\\
\midrule
\citet{brady2021dithering} & $[0,1]$, uniform & Exact average variance & Informal derivation; unbiasedness restricted to $[0,1]$\\
\citet{benbasat2021onebit} & $[0,1]$ & Worst-case MSE & Finite or unrestricted shared randomness; bounds and constructions\\
\citet{feng2026hadamard} & $\R$, vector sources & Distortion after rotation & Full-line unbiased scalar component; different objective\\
This work & General $P$; Gaussian family & Exact source/minimax second moment & All-code converse, uniqueness, tail necessity, finite-randomness tradeoff, optimization\\
\bottomrule
\end{tabular}
\end{table}

\section{Problem Formulation and Fixed-Source Optimum}
\label{sec:fixed-source}

A scalar public-coin one-bit code consists of public randomness $U\perp x$, an encoder producing $M\in\{0,1\}$, and a reconstruction $\widehat x$.  Decoder randomization may be replaced by its conditional mean, preserving unbiasedness and weakly decreasing the second moment.  Encoder private randomness is represented by
\[
p_u(x):=\Pp(M=1\mid x,U=u)\in[0,1].
\]
For fixed $u$, write the decoder levels as $a_u$ and $a_u+d_u$, so the conditional mean reconstruction is
\begin{equation}
q_u(x)=a_u+d_up_u(x).
\label{eq:general-code}
\end{equation}
A code is pointwise unbiased on a set $D$ if $\E[\widehat x\mid x]=x$ for every $x\in D$.  For $X\sim P$, define the source-averaged second-moment risk $V_P(\cC)=\E\widehat X^2$.  Under pointwise unbiasedness on $\supp(P)$,
\begin{equation}
\E(\widehat X-X)^2=V_P(\cC)-\E X^2.
\label{eq:mse-secondmoment}
\end{equation}

For a general Borel source, let $F$ be its right-continuous CDF and let $I_P$ denote the convex hull of its support.

\begin{theorem}[Fixed-source second-moment optimum]
\label{thm:fixed-source}
Let $P$ be a Borel probability distribution on $\R$ with mean $m$, finite second moment, and $0\le \cJ(P)<\infty$.  Then
\begin{equation}
\inf_{\cC\textnormal{ pointwise unbiased on }\R}V_P(\cC)
=m^2+\cJ(P)^2.
\label{eq:fixed-source-value}
\end{equation}
If $\cJ(P)=0$, then $P$ is degenerate at $m$ and the value is attained by a separate anchored full-line code.  Suppose henceforth that $0<\cJ(P)<\infty$.  The same value is attained among codes pointwise unbiased on $I_P$.  If $I_P=\R$ and $0<F(r)<1$ for every $r\in\R$, the full-line infimum is attained by drawing
\begin{equation}
h_P(r)=\frac{\sqrt{F(r)(1-F(r))}}{\cJ(P)},
\label{eq:optimal-density}
\end{equation}
sending $M=\1\{x\ge R\}$, and decoding
\begin{equation}
\widehat x=m+\frac{M-(1-F(R))}{h_P(R)}.
\label{eq:optimal-decoder}
\end{equation}
For bounded or one-sided support, full-line codes approach \eqref{eq:fixed-source-value} arbitrarily closely by adding an arbitrarily small full-support threshold component.
\end{theorem}

Achievability follows from the tail identity
\[
\int_{\R}\bigl(\1\{x\ge r\}-(1-F(r))\bigr)\,\dd r=x-m
\]
and importance sampling.  The converse uses the monotone witness
\begin{equation}
G_P(x)=\gamma(F(x)),\qquad
\gamma(u)=\frac{2u-1}{2\sqrt{u(1-u)}}
\label{eq:dual-witness}
\end{equation}
for continuous sources, and a distributional-transform version for atoms.  A sharp bathtub inequality gives
\begin{equation}
\left|\E[G_P(X)p(X)]\right|\le\sqrt{\alpha(1-\alpha)},
\qquad \alpha=\E p(X),
\label{eq:bathtub-main}
\end{equation}
for every randomized binary encoder $p:\R\to[0,1]$.  Together with $\E[XG_P(X)]=\cJ(P)$, this yields the global lower bound.  Appendix~\ref{app:fixed-source} gives the complete argument for atoms, flat CDF regions, and truncation.

We refer to \eqref{eq:optimal-density}--\eqref{eq:optimal-decoder} as the \emph{distribution-centered random-threshold code}.

\paragraph{Interpretation of the square-root threshold law.}
For each threshold $r$, the centered comparison atom
$A_r(x)=\1\{x\ge r\}-(1-F(r))$ has mean zero under $P$ and
$\|A_r\|_{L_2(P)}=\sqrt{F(r)(1-F(r))}$.  The tail identity represents
$x-m=\int_{\R}A_r(x)\,\dd r$.  Sampling one atom with density $h$ and importance-weighting it gives excess second moment
\[
\int_{\R}\frac{F(r)(1-F(r))}{h(r)}\,\dd r.
\]
Cauchy--Schwarz therefore selects $h(r)\propto\sqrt{F(r)(1-F(r))}$ within the threshold family.  The dual witness in \eqref{eq:dual-witness} is generated by the same $L_2(P)$ geometry, explaining why the converse over arbitrary randomized binary encoders meets this constructive bound exactly rather than only up to order.

\begin{corollary}[Uniform-source special case]
\label{cor:uniform}
For $P=\operatorname{Unif}[0,1]$, among codes unbiased on $[0,1]$,
\[
h_P(r)=\frac8\pi\sqrt{r(1-r)},\qquad
\inf\E(\widehat X-X)^2=\frac{\pi^2}{64}-\frac1{12}.
\]
The two reconstruction levels are
\[
\frac12+\frac\pi8\sqrt{\frac r{1-r}}
\quad\text{and}\quad
\frac12-\frac\pi8\sqrt{\frac{1-r}{r}},
\]
recovering the formulas of \citet{brady2021dithering}.
\end{corollary}

For full-support regular sources, the converse has rigid equality conditions.

\begin{theorem}[Essential uniqueness of the attaining code]
\label{thm:uniqueness}
Let $P$ have a finite second moment, a continuous strictly increasing CDF on $\R$, and $0<\cJ(P)<\infty$.  Every full-line pointwise-unbiased one-bit code attaining $m^2+\cJ(P)^2$ is, up to null sets, bit relabeling, and refinements of the public seed, the threshold code \eqref{eq:optimal-density}--\eqref{eq:optimal-decoder}.
\end{theorem}

The proof identifies the equality conditions in \eqref{eq:bathtub-main} and in Jensen's inequality.  For almost every public state, the encoder must be a deterministic threshold, and the difference between its two decoder levels must equal $\cJ(P)/\sqrt{F(r)(1-F(r))}$.  A Fubini argument supplies a common $P$-full dense set on which this threshold representation holds simultaneously for almost every public state.  Interval identities on that set, followed by monotone continuity of the induced threshold measure, force the weighted threshold measure to be Lebesgue measure and uniquely determine $h_P$.

\section{Gaussian Minimax Optimality and the Tail Price}
\label{sec:gaussian}

Consider $X_\mu\sim\cN(\mu,\sigma^2)$ with $|\mu|\le c\sigma$.  Let $\Phi$ and $\phi$ denote the standard normal CDF and density, and write $\barPhi=1-\Phi$.  Standardize by $\sigma$ and define
\begin{equation}
P_c^\star=\tfrac12\cN(-c,1)+\tfrac12\cN(c,1),\qquad
F_c(z)=\tfrac12\Phi(z-c)+\tfrac12\Phi(z+c),
\label{eq:endpoint-mixture}
\end{equation}
\begin{equation}
\Lambda_c=\int_{\R}\sqrt{F_c(z)(1-F_c(z))}\,\dd z,
\qquad
\rho_c(z)=\frac{\sqrt{F_c(z)(1-F_c(z))}}{\Lambda_c}.
\label{eq:lambda-rho}
\end{equation}

\begin{theorem}[Exact Gaussian location minimax value]
\label{thm:gaussian-minimax}
Over all full-line pointwise-unbiased public-coin one-bit codes,
\begin{equation}
\inf_{\cC}\sup_{|\mu|\le c\sigma}\E_\mu\widehat X^2
=\sigma^2\Lambda_c^2.
\label{eq:gaussian-minimax-value}
\end{equation}
An optimal code draws $Z\sim\rho_c$, sends $M=\1\{X/\sigma\ge Z\}$, and returns
\begin{equation}
\widehat X=\sigma\frac{M-(1-F_c(Z))}{\rho_c(Z)}.
\label{eq:gaussian-code}
\end{equation}
The worst-case second-moment risk is attained at $\mu=\pm c\sigma$.  Moreover, every attaining minimax code is equivalent to \eqref{eq:gaussian-code} in the sense of Theorem~\ref{thm:uniqueness} for the endpoint mixture.
\end{theorem}

The lower bound averages the two endpoint risks and applies Theorem~\ref{thm:fixed-source} to $P_c^\star$.  For the upper bound, thresholds $z$ and $-z$ are paired.  Their paired contribution to the second-moment integrand is nondecreasing in $a=|\mu|/\sigma$ because
\[
\partial_a H_a(z)=(1-2q(z))\bigl(\phi(z-a)-\phi(z+a)\bigr)\ge0,
\qquad q(z)=1-F_c(z),\ z\ge0.
\]
Thus the Bayes and worst-case bounds coincide.

\begin{table}[H]
\centering
\caption{Exact scalar minimax constants.}
\label{tab:constants}
\begin{tabular}{c@{\qquad}cc}
\toprule
$c$ & $\Lambda_c$ & $\Lambda_c^2$\\
\midrule
$0$ & $1.6147438534$ & $2.6073977122$\\
$1/2$ & $1.7938748654$ & $3.2179870327$\\
$1$ & $2.1943083796$ & $4.8149892647$\\
\bottomrule
\end{tabular}
\end{table}

Exact second-moment optimality has a sharp tail cost.  Let $W_{a,c}$ denote the standardized output of \eqref{eq:gaussian-code} under $\cN(a,1)$.

\begin{theorem}[Moment phase transition and critical endpoint tail]
\label{thm:tail-phase}
For every $p>0$ and $|a|\le c$,
\[
\E|W_{a,c}|^p<\infty
\quad\Longleftrightarrow\quad
\bigl[p<3\bigr]\ \text{or}\ \bigl[p=3\text{ and }|a|<c\bigr].
\]
At $a=c$,
\begin{equation}
\Pp(W_{c,c}>t)=\Theta\!\left(\frac{1}{t^3\sqrt{\log t}}\right)
\qquad (t\to\infty),
\label{eq:critical-tail}
\end{equation}
and the reflected statement holds at $a=-c$.  By Theorem~\ref{thm:uniqueness}, every attaining Gaussian minimax code has the same failure of third and higher absolute moments at the endpoint means.
\end{theorem}

Theorem~\ref{thm:tail-phase} separates expected-risk optimality from concentration: the minimax code has finite second moment but an infinite third absolute moment at the endpoint means.  Its tail can be regularized at arbitrarily small second-moment cost.  Let $r(z)=1/[\pi(1+z^2)]$ and
\[
\rho_{c,\eta}(z)=(1-\eta)\rho_c(z)+\eta r(z),\qquad 0<\eta<1.
\]
Use the decoder \eqref{eq:gaussian-code} with $\rho_c$ replaced in the denominator and sampling law by $\rho_{c,\eta}$, while retaining the same centering $1-F_c(Z)$.  This mixes the optimal threshold law with a Cauchy density and thereby imposes a polynomial lower envelope on the threshold density; we refer to the resulting construction as the \emph{Cauchy-mixture code}.

\begin{corollary}[Cauchy-mixture robustification]
\label{cor:robustified}
The perturbed code is pointwise unbiased on $\R$ and satisfies
\begin{equation}
\sup_{|\mu|\le c\sigma}\E_\mu\widehat X_\eta^2
\le\frac{\sigma^2\Lambda_c^2}{1-\eta}.
\label{eq:robust-risk}
\end{equation}
For every $p>0$,
$\sup_{|\mu|\le c\sigma}\E_\mu|\widehat X_\eta|^p<\infty$.
\end{corollary}

\paragraph{Sensitivity to scale misspecification.}
Pointwise unbiasedness is preserved under scale misspecification, although moment finiteness need not be.  Appendix~\ref{app:robustification} gives the sharp $\sqrt{2}$ threshold for the attaining code and shows that the Cauchy-mixture code retains all positive-order absolute moments under every finite mismatch.

\section{Finite Public Randomness and Bounded Outputs}
\label{sec:finite}

The exact construction samples its threshold from a continuous distribution with infinite support.  We next quantify what changes when shared randomness and decoder outputs are required to be finite.  The following proposition identifies the obstruction to exact pointwise unbiasedness on all of $\R$.

\begin{proposition}[Finite-support impossibility under finite decoder means]
\label{prop:finite-impossible}
Suppose the public random variable has finite support and each decoder branch has a finite conditional mean.  Then no one-bit code can satisfy $\E[\widehat x\mid x]=x$ for every $x\in\R$.
\end{proposition}

Indeed, the conditional expectation lies in the convex hull of finitely many decoder means and is therefore bounded.  Proposition~\ref{prop:finite-impossible} motivates an approximate implementation that uses finitely many shared random bits and bounded decoder outputs while retaining uniform control over the low-SNR Gaussian family.

Fix $L>c$, an even integer $K$, $\Delta=2L/K$, and midpoint thresholds
\[
r_i=-L+(i-\tfrac12)\Delta,\qquad q_i=1-F_c(r_i),\qquad
S_{K,L}=\Delta\sum_{i=1}^K\sqrt{q_i(1-q_i)}.
\]
The ideal threshold probabilities are $\pi_i^\star=\Delta\sqrt{q_i(1-q_i)}/S_{K,L}$.  Let $N=2^R>K$ and choose integers $n_i\ge1$, $\sum_i n_i=N$, such that
\begin{equation}
\pi_i:=\frac{n_i}{N}\ge\left(1-\frac KN\right)\pi_i^\star.
\label{eq:prob-rounding}
\end{equation}
Such counts always exist.  A uniform shared-randomness seed of $R$ bits selects threshold $i$ with probability $\pi_i$; the encoder sends $M=\1\{X/\sigma\ge r_i\}$ and the decoder returns
\begin{equation}
\widehat X_{K,L,R}=\sigma\frac{\Delta}{\pi_i}(M-q_i).
\label{eq:finite-code}
\end{equation}

\begin{theorem}[Finite-randomness, bounded-output approximation]
\label{thm:finite-coin}
Let $t=L-c>0$.  The code \eqref{eq:finite-code} communicates one bit, uses exactly $R$ shared random bits, stores $K$ threshold/decoder entries and the counts $n_1,\ldots,n_K$ (or an equivalent cumulative seed-to-threshold map), and satisfies
\begin{align}
\sup_{|\mu|\le c\sigma}
\left|\E_\mu\widehat X_{K,L,R}-\mu\right|
&\le\sigma\left[\frac LK+2\bigl(\phi(t)-t\barPhi(t)\bigr)\right],
\label{eq:finite-bias}\\
\sup_{|\mu|\le c\sigma}\E_\mu\widehat X_{K,L,R}^2
&\le\frac{\sigma^2S_{K,L}^2}{1-K/N},
\label{eq:finite-risk}\\
|\widehat X_{K,L,R}|&\le\sigma\Delta N.
\label{eq:finite-amplitude}
\end{align}
If $N\ge K(1+\eta)/\eta$, then the multiplicative second-moment inflation in \eqref{eq:finite-risk} is at most $1+\eta$.  Moreover $S_{K,L}\to\Lambda_c$ as first $K\to\infty$ and then $L\to\infty$.  Rounding every stored decoder level to absolute precision $\tau\sigma$ adds at most $\tau\sigma$ to the bias and at most $\tau\sigma$ to the root second moment.
\end{theorem}

For later use, define the standardized bias and second-moment constants
\begin{equation}
\beta_{K,L}:=\frac LK+2\bigl(\phi(L-c)-(L-c)\barPhi(L-c)\bigr),
\qquad
V_{K,L,R}:=\frac{S_{K,L}^2}{1-K/2^R}.
\label{eq:finite-constants}
\end{equation}
Thus $R=O(\log K+\log(1/\eta))$ shared random bits suffice for a bounded-output code approaching the minimax constant.

\begin{remark}[Explicit parameter selection]
\label{rem:finite-parameter-selection}
Let $\psi_c(z)=\sqrt{F_c(z)(1-F_c(z))}$ and
$D_{c,L}=\sup_{|z|\le L}|\psi_c'(z)|$.  For $L-c\ge1$, Appendix~\ref{app:finite} proves
\begin{equation}
|S_{K,L}-\Lambda_c|
\le
\frac{D_{c,L}L^2}{K}
+4(2\pi)^{-1/4}(L-c)^{-3/2}e^{-(L-c)^2/4}.
\label{eq:explicit-finite-rate}
\end{equation}
Consequently,
$\sqrt{V_{K,L,R}}\le(\Lambda_c+\varepsilon_{K,L})/\sqrt{1-K/2^R}$,
where $\varepsilon_{K,L}$ is the right-hand side of \eqref{eq:explicit-finite-rate}.  For a target scalar approximation level $\zeta\in(0,1)$, one may take $L-c=\Theta(\sqrt{\log(1/\zeta)})$, choose the next even integer $K\gtrsim\max\{L,D_{c,L}L^2\}/\zeta$, and finally set
$R\ge\lceil\log_2(K(1+\eta)/\eta)\rceil$.
\end{remark}

\begin{table}[H]
\centering
\caption{Tradeoffs among the exact, robustified, and finite-randomness scalar codes.}
\label{tab:code-tradeoffs}
\scriptsize
\setlength{\tabcolsep}{3pt}
\renewcommand{\arraystretch}{1.12}
\begin{tabular}{@{}>{\raggedright\arraybackslash}p{0.15\textwidth}>{\raggedright\arraybackslash}p{0.19\textwidth}>{\raggedright\arraybackslash}p{0.22\textwidth}>{\raggedright\arraybackslash}p{0.19\textwidth}>{\raggedright\arraybackslash}p{0.17\textwidth}@{}}
\toprule
Code & Unbiasedness & Public randomness / output & Endpoint absolute moments & Second-moment guarantee\\
\midrule
Exact minimax & Exact on all of $\R$ & Continuous, infinite-support; unbounded output & Finite exactly for $p<3$ & $\sigma^2\Lambda_c^2$\\
Cauchy-mixture robustified & Exact on all of $\R$ & Continuous, infinite-support; unbounded output & Finite for every $p>0$ & $\le\sigma^2\Lambda_c^2/(1-\eta)$\\
Finite-randomness grid & Gaussian-family bias $\le\sigma\beta_{K,L}$ & Exactly $R$ shared random bits; bounded output & All positive orders finite & $\le\sigma^2V_{K,L,R}$\\
\bottomrule
\end{tabular}
\end{table}

Figure~\ref{fig:diagnostics} illustrates endpoint maximality, the order-three tail transition, and finite-grid convergence.

\begin{figure}[!htbp]
\centering
\includegraphics[width=0.95\textwidth]{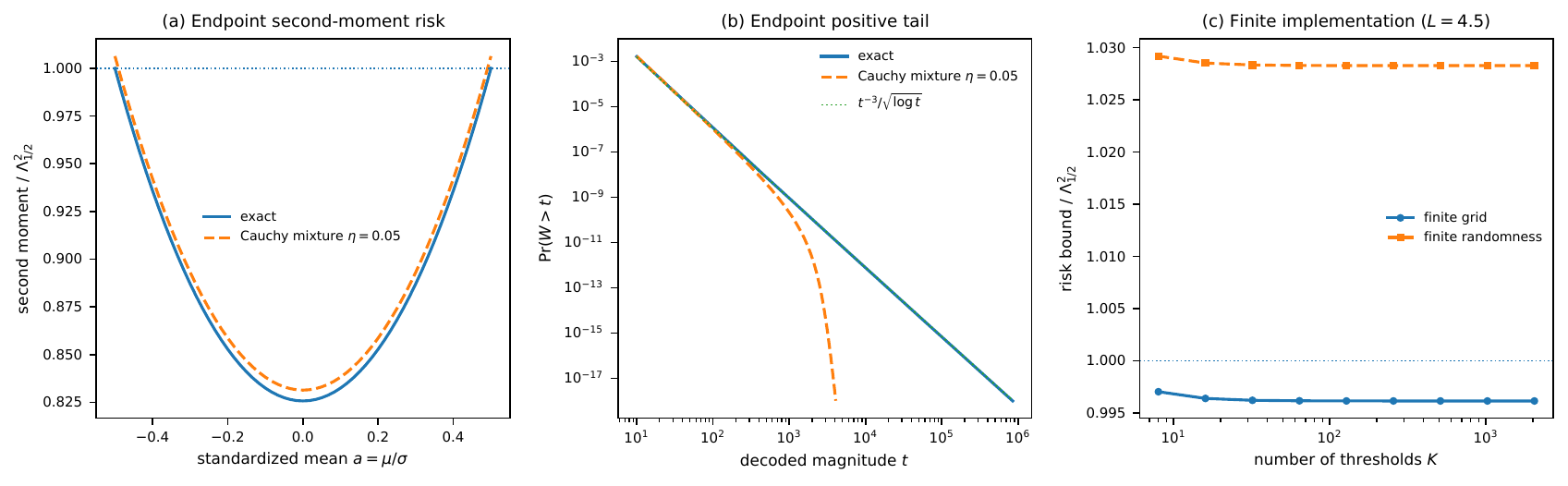}
\caption{Numerical diagnostics for $c=1/2$.  (a) Exact-code second-moment risk increases with $|\mu|$ and peaks at the endpoints.  (b) The exact endpoint survival has the predicted $t^{-3}/\sqrt{\log t}$ critical tail, while the Cauchy-mixture robustification is substantially lighter tailed.  (c) Finite-grid second-moment constants converge to $\Lambda_{1/2}^2$; each point stores $K$ thresholds and uses $R=\lceil\log_2 N\rceil$ shared random bits, with $N$ the next power of two above $21K$.  The finite-grid curve may fall below one because the finite implementation is only approximately unbiased; Theorem~\ref{thm:finite-coin} controls the corresponding bias.}
\label{fig:diagnostics}
\end{figure}

\section{Fixed-Bit Gaussian-Gradient Compression and Optimization}
\label{sec:optimization}

\subsection{Fixed-bit vector compression}

Let $B,d\in\mathbb N$ with $B,d\ge1$, and set $s=\min\{B,d\}$.  For $g\sim\cN(\mu,\sigma^2I_d)$ with $\|\mu\|_\infty\le c\sigma$, fresh public randomness selects a uniform subset $S\subseteq[d]$ of size $s$.  On each selected coordinate, the encoder and decoder run an independent scalar code, and the vector decoder returns
\begin{equation}
[\cQ_B(g)]_j=\frac ds\1\{j\in S\}\widehat g_j^{\rm sc}.
\label{eq:vector-compressor}
\end{equation}
The subset and thresholds are public.  If $B>d$, the encoder appends $B-d$ deterministic padding bits, so the message length is exactly $B$.

\begin{corollary}[Fixed-bit compression of unbounded Gaussian gradients]
\label{cor:vector-compressor}
For every deterministic $g\in\R^d$,
\[
\E[\cQ_B(g)\mid g]=g.
\]
Moreover, if $g\sim\cN(\mu,\sigma^2I_d)$ with $\|\mu\|_\infty\le c\sigma$, then
\begin{equation}
\E\|\cQ_B(g)\|_2^2
\le\Lambda_c^2\frac{\sigma^2d^2}{\min\{B,d\}}.
\label{eq:vector-secondmoment}
\end{equation}
Replacing each scalar component with the robustified code of Corollary~\ref{cor:robustified} multiplies the right-hand side by at most $1/(1-\eta)$ and makes all positive-order coordinatewise absolute moments finite.
\end{corollary}

\subsection{Exact-code guarantees on Kim's continuous quadratic hard family}

Consider Kim's continuous quadratic hard family \citep{kim2026bitconstrained}:
\begin{equation}
f_\theta(x)=\tfrac12\|x-\theta\|_2^2,
\quad \theta\in[-\delta,\delta]^d,
\quad g_t=x_t-\theta+\xi_t,
\quad \xi_t\sim\cN(0,\sigma^2I_d).
\label{eq:quadratic-family}
\end{equation}
At the parameter scale used in Kim's lower bound,
\begin{equation}
\delta^2=\frac{4\pi^2\eps^2}{d},\qquad
\eps^2\le\frac{\sigma^2d}{16\pi^2},
\label{eq:hard-scale}
\end{equation}
so $\delta\le\sigma/2$.  The estimator queries $x_t=0$ at every round, decodes $\widetilde g_t=\cQ_B(g_t)$ with $c=1/2$, averages the decoded vectors, negates the result, and projects it onto the cube.

\begin{corollary}[Matching expected rate on the hard family]
\label{cor:quadratic}
The estimator above satisfies
\begin{equation}
\sup_\theta\E\bigl[f_\theta(\widehat\theta)-f_\theta^\star\bigr]
\le\frac{\Lambda_{1/2}^2}{2}
\frac{\sigma^2d^2}{T\min\{B,d\}}.
\label{eq:quadratic-upper}
\end{equation}
Consequently it attains \eqref{eq:target-rate}, matching the lower bound in its dependence on $(\sigma,d,B,\eps)$.
\end{corollary}

At the endpoint means, the attaining code has an infinite third absolute moment, but finite variance is sufficient for robust estimation.  To obtain a high-probability guarantee from finite variance alone, we use a deterministic balanced coordinate schedule and coordinatewise median-of-means aggregation.

\begin{corollary}[High-probability quadratic guarantee]
\label{cor:quadratic-hp}
There is a universal constant $C$ such that, for every $\delta_0\in(0,1)$ and
$T\min\{B,d\}\ge C d\log(2d/\delta_0)$, a protocol that communicates exactly $B$ bits per query satisfies, with probability at least $1-\delta_0$,
\begin{equation}
f_\theta(\widehat\theta)-f_\theta^\star
\le C\Lambda_{1/2}^2
\frac{\sigma^2d^2}{T\min\{B,d\}}
\log\frac{2d}{\delta_0}.
\label{eq:quadratic-hp}
\end{equation}
\end{corollary}

\subsection{End-to-end finite-randomness implementation}

\begin{corollary}[End-to-end finite-randomness optimization guarantee]
\label{cor:finite-quadratic}
Set $c=1/2$, $s=\min\{B,d\}$, and use the finite-randomness scalar code of Theorem~\ref{thm:finite-coin} on a deterministic balanced coordinate schedule.  Suppose $d$ divides $Ts$; otherwise fewer than $d/\gcd(d,s)$ additional rounds make this true.  Then every query communicates exactly $B$ bits, uses $sR$ shared random bits, and produces bounded decoder outputs.  The resulting quadratic estimator satisfies
\begin{equation}
\sup_\theta\E\bigl[f_\theta(\widehat\theta_{K,L,R})-f_\theta^\star\bigr]
\le
\frac{\sigma^2V_{K,L,R}d^2}{2T\min\{B,d\}}
+\frac{d\sigma^2\beta_{K,L}^2}{2}.
\label{eq:finite-quadratic}
\end{equation}
If each stored decoder level is rounded to absolute precision $\tau\sigma$, the same statement holds with $\beta_{K,L}$ replaced by $\beta_{K,L}+\tau$ and $V_{K,L,R}$ replaced by $(\sqrt{V_{K,L,R}}+\tau)^2$.
\end{corollary}

Equation~\eqref{eq:finite-quadratic} separates the sampling-variance term from the squared bias introduced by the finite approximation.

\section{Discussion}

The results distinguish three phenomena.  First, unbounded Gaussian observations do not alter the expected query-complexity rate on the localized quadratic family.  Second, exact scalar second-moment optimality entails the endpoint order-three absolute-moment transition.  Third, exact pointwise unbiasedness on all of $\R$ is incompatible with finite public support and finite decoder means, and also with uniformly bounded decoder outputs.  Corollary~\ref{cor:robustified} and Theorem~\ref{thm:finite-coin} quantify two corresponding relaxations: a small second-moment slack restores all positive-order absolute moments, whereas finite shared randomness and bounded outputs introduce a controlled approximation bias.

Open directions include the optimal constant for joint $B$-bit vector coding, the joint bias--variance--randomness--storage tradeoff, and localization mechanisms for general strongly-convex objectives.

\section{Conclusion}

We characterized the source-dependent optimum of public-coin, pointwise-unbiased one-bit compression, identified a least-favorable prior for the Gaussian family, and showed that attaining the Gaussian minimax optimum forces an order-three absolute-moment transition.  The Cauchy-mixture and finite-randomness constructions quantify complementary tradeoffs among second moment, higher moments, bias, and shared-randomness complexity.  Coordinate allocation then gives a fixed-bit Gaussian-gradient scheme with lower-bound-matching parameter dependence on Kim's continuous quadratic hard family.

\section*{Reproducibility Statement}
The supplementary code archive contains arbitrary-precision evaluations of $\Lambda_c$, scripts for generating every panel of Figure~\ref{fig:diagnostics}, finite-grid and probability-rounding checks, unit tests, and exact communication accounting.  The file \texttt{reproduce\_all.sh} regenerates the reported numerical outputs and runs the full test suite.  No numerical computation enters any theorem proof.

\begingroup\raggedright
\ifanonymoussubmission
An anonymized code archive is included in the supplementary material.
\if\relax\detokenize\expandafter{\CodeRepositoryURL}\relax\else
An anonymous repository mirror is available at \url{\CodeRepositoryURL}.
\fi
\else
\if\relax\detokenize\expandafter{\CodeRepositoryURL}\relax
Code and exact reproduction instructions are included in the accompanying release archive.
\else
Code and exact reproduction instructions are available at \href{\CodeRepositoryURL}{\texttt{github.com/jiangxioabai/onebit\_unbiased\_compression}}.
\fi
\fi
The archive reproduces Table~\ref{tab:constants}, all panels of Figure~\ref{fig:diagnostics}, and the finite-grid calculations used in Theorem~\ref{thm:finite-coin}.  The numerical scripts are diagnostic only.
\par\endgroup

\section*{AI Use Statement}
Generative AI tools were used to explore candidate proof strategies, check algebra, assist with literature triage and code debugging, and edit prose.
\ifanonymoussubmission
\else
The authors independently verified all mathematical claims, citations, and numerical outputs; complete proofs and reproduction scripts are provided.  The authors are responsible for the final manuscript.
\fi

%% file: appendix.tex
\section{Code-Class Reductions and Distributional Preliminaries}
\label{app:prelim}

We first record reductions used throughout the proofs.

\begin{lemma}[Decoder randomization is unnecessary]
\label{lem:decoder-randomization}
For any public-coin one-bit code with randomized decoder, replacing the output by its conditional mean given $(U,M)$ preserves the conditional mean for every input and weakly decreases the source-averaged second moment.
\end{lemma}

\begin{proof}
Let $\widehat X$ be the original output and put $\widetilde X=\E[\widehat X\mid U,M]$.  Iterated expectation preserves pointwise unbiasedness.  Conditional Jensen gives $\E\widetilde X^2\le\E\widehat X^2$.
\end{proof}

After this reduction, for fixed $U=u$ the decoder has two values $a_u$ and $a_u+d_u$.  Encoder private randomness is summarized by $p_u(x)=\Pp(M=1\mid x,U=u)$.  If $X\sim P$ and $\alpha_u=\E_Pp_u(X)$, then
\begin{equation}
\E[\widehat X^2\mid U=u]
=(a_u+d_u\alpha_u)^2+d_u^2\alpha_u(1-\alpha_u).
\label{eq:conditional-secondmoment-app}
\end{equation}

For a Borel distribution $P$, write $F(x^-)=\lim_{y\uparrow x}F(y)$ and $\Delta F(x)=F(x)-F(x^-)$.  If $V\sim\operatorname{Unif}[0,1]$ is independent of $X\sim P$, the distributional transform
\begin{equation}
U_P=F(X^-)+V\Delta F(X)
\label{eq:distributional-transform}
\end{equation}
is uniform on $[0,1]$.  Let
\[
\psi(u)=\sqrt{u(1-u)},\qquad
\gamma(u)=-\psi'(u)=\frac{2u-1}{2\sqrt{u(1-u)}}.
\]
Define the generalized witness
\begin{equation}
G_P(X)=\E[\gamma(U_P)\mid X].
\label{eq:generalized-witness}
\end{equation}
For continuous strictly increasing $F$, this reduces to $G_P(x)=\gamma(F(x))$.
For $K>0$, also define the bounded odd truncations
\begin{equation}
\gamma_K(u)=\clip(\gamma(u),-K,K),
\qquad
G_{P,K}(X)=\E[\gamma_K(U_P)\mid X].
\label{eq:truncated-witness}
\end{equation}
Both $\gamma$ and $\gamma_K$ have mean zero under a uniform argument.

\begin{lemma}[Randomized bathtub inequality]
\label{lem:bathtub}
For every measurable $p:\R\to[0,1]$ and $\alpha=\E_Pp(X)$,
\begin{equation}
\left|\E_P[G_P(X)p(X)]\right|
\le\sqrt{\alpha(1-\alpha)}.
\label{eq:bathtub-app}
\end{equation}
The same upper bound holds with $G_P$ replaced by every $G_{P,K}$.  If $F$ is continuous and strictly increasing and $0<\alpha<1$, equality in \eqref{eq:bathtub-app} holds only when $p$ is almost surely the indicator of an upper or lower quantile interval, up to a null set.
\end{lemma}

\begin{proof}
By \eqref{eq:distributional-transform},
\[
\E[G_P(X)p(X)]=\E[\gamma(U_P)p(X)].
\]
Conditioning on $U_P=u$ writes the right-hand side as $\int_0^1\gamma(u)q(u)\,\dd u$ for some $q:[0,1]\to[0,1]$ with $\int_0^1q=\alpha$.  Since $\gamma$ is strictly increasing, the bathtub principle gives
\[
\int_0^1\gamma(u)q(u)\,\dd u
\le\int_{1-\alpha}^1\gamma(u)\,\dd u
=\sqrt{\alpha(1-\alpha)}.
\]
The lower bound is obtained from the interval $[0,\alpha]$.  Strict monotonicity gives the equality characterization in the continuous case.

For $\gamma_K$, the same bathtub argument reduces the maximum to the upper interval.  If $\alpha\le1/2$, clipping only decreases the nonnegative integrand on that interval.  If $\alpha>1/2$, odd symmetry and zero mean rewrite its upper-interval integral as minus the lower-interval integral of length $1-\alpha$, whose magnitude also decreases under clipping.  Thus its absolute optimum is no larger than $\sqrt{\alpha(1-\alpha)}$.
\end{proof}

\begin{lemma}[Dual correlation identity]
\label{lem:dual-correlation}
If $\E X^2<\infty$ and $\cJ(P)<\infty$, then
\begin{equation}
\E_P[XG_P(X)]=\cJ(P).
\label{eq:dual-correlation-app}
\end{equation}
\end{lemma}

\begin{proof}
Let $Q(u)=\inf\{x:F(x)\ge u\}$ be the generalized quantile.  The distributional transform couples $X$ and $U_P$ so that $X=Q(U_P)$ almost surely.  Hence
\[
\E[XG_P(X)]=\int_0^1Q(u)\gamma(u)\,\dd u.
\]
For the clipped quantile $Q_M=\clip(Q,-M,M)$, Stieltjes integration by parts gives
\[
\int_0^1Q_M(u)\gamma(u)\,\dd u
=\int_{(0,1)}\psi(u)\,\dd Q_M(u).
\]
The quantile change-of-variables identity for monotone functions yields
\[
\int_{(0,1)}\psi(u)\,\dd Q_M(u)
=\int_{-M}^{M}\sqrt{F(r)(1-F(r))}\,\dd r
\]
up to endpoint terms that vanish as $M\to\infty$.  This identity is immediate for finite-support distributions, and the general case follows by monotone approximation of $Q_M$.

We make the tail integrability explicit.  Since $Q$ is monotone and $Q\in L_2(0,1)$,
\[
\sqrt{1-u}\,Q(u)\longrightarrow0\quad(u\uparrow1),
\qquad
\sqrt{u}\,Q(u)\longrightarrow0\quad(u\downarrow0).
\]
For example, if $Q$ is eventually positive near one, then
$(1-u)Q(u)^2\le\int_u^1Q(v)^2\,\dd v\to0$; if it is not, monotonicity makes it bounded on that tail.  The lower endpoint is analogous.  Monotonicity also implies that $Q(u)\gamma(u)$ has a fixed sign on sufficiently small neighborhoods of zero and one.  On each such tail, integration by parts identifies the absolute integral of $Q\gamma$ with a finite boundary term plus the corresponding tail contribution to $\int\psi\,\dd Q=\cJ(P)$.  On the remaining compact interval, $\gamma$ is bounded and $Q\in L_1(0,1)$.  Hence $Q\gamma$ is absolutely integrable.  Therefore
\[
\E[XG_{P,K}(X)]
=\int_0^1Q(u)\gamma_K(u)\,\dd u
\longrightarrow
\int_0^1Q(u)\gamma(u)\,\dd u
=\cJ(P)
\]
by dominated convergence.  This also justifies the preceding $M\to\infty$ limit and proves the lemma.
\end{proof}

\section{Proof of the Fixed-Source Characterization}
\label{app:fixed-source}

\subsection{Global converse for arbitrary Borel sources}

Fix any code pointwise unbiased on $\R$ and use Lemma~\ref{lem:decoder-randomization}.  For each public state $u$, let
\[
\alpha_u=\E p_u(X),\qquad
m_u=a_u+d_u\alpha_u,
\qquad
c_{u,K}=d_u\E[G_{P,K}(X)p_u(X)].
\]
Lemma~\ref{lem:bathtub} and \eqref{eq:conditional-secondmoment-app} give, for every $K$,
\begin{equation}
\E[\widehat X^2\mid U=u]
\ge m_u^2+c_{u,K}^2.
\label{eq:local-two-dimensional-truncated}
\end{equation}
Pointwise unbiasedness gives $\E_Um_u=m$.  Because $G_{P,K}$ is bounded and centered, Fubini is now immediate and yields
\begin{align}
\E_Uc_{u,K}
&=\E\!\left[G_{P,K}(X)\E_Uq_U(X)\right]\\
&=\E[XG_{P,K}(X)]
=:\cJ_K(P).
\label{eq:truncated-correlation}
\end{align}
Consequently, Jensen's inequality in $\R^2$ gives
\begin{equation}
V_P(\cC)\ge m^2+\cJ_K(P)^2
\qquad\text{for every }K.
\label{eq:truncated-converse}
\end{equation}
By Lemma~\ref{lem:dual-correlation} and dominated truncation,
$\cJ_K(P)\to\cJ(P)$.  Letting $K\to\infty$ proves
\begin{equation}
V_P(\cC)\ge m^2+\cJ(P)^2.
\label{eq:fixed-converse}
\end{equation}
This truncation step avoids assuming a priori that the unbounded dual witness can be interchanged with the public-seed integral.

\subsection{Achievability on the convex hull of the support}

We first treat the degenerate case.  If $\cJ(P)=0$ and $P$ were nondegenerate, there would exist $a<b$ such that $P(X\le a)>0$ and $P(X\ge b)>0$.  Then $F(r)(1-F(r))$ would be bounded away from zero on a subinterval of $(a,b)$, contradicting $\cJ(P)=0$.  Hence $P=\delta_m$.  Choose any everywhere-positive density $h_0$ on $\R$, draw $R\sim h_0$, send $M=\1\{x\ge R\}$, and decode
\[
\widehat x=m+\frac{M-\1\{m\ge R\}}{h_0(R)}.
\]
The tail identity gives $\E[\widehat x\mid x]=x$ for every $x\in\R$, while $\widehat X=m$ almost surely under $P=\delta_m$.  Thus the full-line second-moment risk is exactly $m^2$.

Assume now that $0<\cJ(P)<\infty$.  Let $I_P$ be the convex hull of the support and let
\[
h_P(r)=\frac{\sqrt{F(r)(1-F(r))}}{\cJ(P)}
\]
on the interior of $I_P$.  Values of $F$ at atoms do not affect this Lebesgue density.  Draw $R\sim h_P$, send $M=\1\{x\ge R\}$, and decode as in \eqref{eq:optimal-decoder}.  For every $x\in I_P$,
\begin{align}
\E_R[\widehat x\mid x]
&=m+\int_{I_P}\left(\1\{x\ge r\}-(1-F(r))\right)\,\dd r\\
&=x.
\label{eq:interval-tail-identity}
\end{align}
The final equality is the usual positive/negative-part identity, valid on the convex hull.  Under $X\sim P$, conditional on $R=r$, the bit has success probability $1-F(r)$ for Lebesgue-almost every $r$.  Therefore
\begin{align}
\E\widehat X^2
&=m^2+\int_{I_P}\frac{F(r)(1-F(r))}{h_P(r)}\,\dd r\\
&=m^2+\cJ(P)^2.
\label{eq:fixed-achieve}
\end{align}
If $I_P=\R$, this code is pointwise unbiased on all of $\R$.

If $I_P\ne\R$, choose any everywhere-positive density $r_0$ and set
\[
h_{P,\eta}=(1-\eta)h_P+\eta r_0,
\qquad 0<\eta<1,
\]
where $h_P$ is extended by zero outside $I_P$.  Use the same centered threshold decoder with $h_{P,\eta}$.  Since $h_{P,\eta}>0$ on $\R$, the full tail integral now gives pointwise unbiasedness for every real input.  Moreover $F(1-F)=0$ outside $I_P$ and $h_{P,\eta}\ge(1-\eta)h_P$ inside, so
\begin{equation}
V_P(\cC_\eta)
\le m^2+\frac{\cJ(P)^2}{1-\eta}.
\label{eq:full-line-approx-source}
\end{equation}
Letting $\eta\downarrow0$ and combining with \eqref{eq:fixed-converse} proves Theorem~\ref{thm:fixed-source}.

Finally, pointwise unbiasedness on the source support implies
\[
\E(\widehat X-X)^2
=\E\widehat X^2-2\E[X\E(\widehat X\mid X)]+\E X^2
=V_P(\cC)-\E X^2.
\]

\subsection{The uniform-source formulas}

For $P=\operatorname{Unif}[0,1]$,
\[
\cJ(P)=\int_0^1\sqrt{r(1-r)}\,\dd r=\frac\pi8.
\]
Hence $h_P(r)=(8/\pi)\sqrt{r(1-r)}$.  Since $m=1/2$ and $1-F(r)=1-r$, the two decoder levels are
\begin{align*}
\widehat x_1(r)
&=\frac12+\frac{r}{h_P(r)}
=\frac12+\frac\pi8\sqrt{\frac r{1-r}},\\
\widehat x_0(r)
&=\frac12-\frac{1-r}{h_P(r)}
=\frac12-\frac\pi8\sqrt{\frac{1-r}{r}}.
\end{align*}
The optimal average MSE is
\[
\frac14+\frac{\pi^2}{64}-\E X^2
=\frac{\pi^2}{64}-\frac1{12}.
\]
This proves Corollary~\ref{cor:uniform}.

\section{Equality Cases and Essential Uniqueness}
\label{app:uniqueness}

Assume throughout this section that $P$ has a finite second moment, $F$ is continuous and strictly increasing on $\R$, and $0<\cJ(P)<\infty$.  Lemmas~\ref{lem:bathtub} and~\ref{lem:dual-correlation}, together with the bounded witnesses $G_{P,K}$, provide all integrability and limiting statements used below.  We call two codes equivalent if, after possibly relabeling the two messages and refining the public seed by independent randomness, they agree outside a set of zero public-source measure.

\begin{proof}[Proof of Theorem~\ref{thm:uniqueness}]
Suppose a code attains $m^2+\cJ(P)^2$.  Define
\[
c_u=d_u\E[G_P(X)p_u(X)].
\]
For almost every $u$, dominated convergence under $P$ gives $c_{u,K}\to c_u$.  Moreover, \eqref{eq:local-two-dimensional-truncated} implies $|c_{u,K}|^2\le \E[\widehat X^2\mid U=u]$.  Hence dominated convergence in $L_2(U)$ gives
\[
\E_Uc_u=\lim_{K\to\infty}\E_Uc_{u,K}=\cJ(P),
\qquad
\E[\widehat X^2\mid U=u]\ge m_u^2+c_u^2.
\]
Thus the limiting two-dimensional lower-bound chain is exact.  Decoder-side conditional variance must vanish by Lemma~\ref{lem:decoder-randomization}, and equality in Jensen implies
\begin{equation}
m_u=m,
\qquad
c_u=\cJ(P)
\label{eq:jensen-equality}
\end{equation}
for almost every $u$.  In particular, $0<\alpha_u<1$ and $d_u\ne0$ almost surely.

Equality in the limiting local inequality forces equality in Lemma~\ref{lem:bathtub}.  Since $F$ and $\gamma\circ F$ are strictly increasing, $p_u$ is almost surely an upper or lower threshold indicator.  Relabeling the bit when necessary, write
\begin{equation}
p_u(x)=\1\{x\ge r_u\},\qquad d_u>0.
\label{eq:equality-threshold}
\end{equation}
Then $\alpha_u=1-F(r_u)$ and
\[
\E[G_P(X)p_u(X)]=\sqrt{F(r_u)(1-F(r_u))}.
\]
Equation \eqref{eq:jensen-equality} therefore gives
\begin{equation}
d_u=\frac{\cJ(P)}{\sqrt{F(r_u)(1-F(r_u))}},
\qquad
a_u=m-d_u(1-F(r_u)).
\label{eq:equality-levels}
\end{equation}

It remains to identify the threshold law.  The threshold location can be chosen measurably as $r_u=F^{-1}(1-\alpha_u)$.  Let
\[
E=\{(u,x):p_u(x)\ne\1\{x\ge r_u\}\}.
\]
Equation \eqref{eq:equality-threshold} gives $(P_U\otimes P)(E)=0$.  Fubini therefore yields a common set $D\subseteq\R$ with $P(D)=1$ such that, for every $x\in D$, the threshold identity holds for $P_U$-almost every $u$.  Since $F$ is strictly increasing, every nonempty open interval has positive $P$-mass, so $D$ is dense in $\R$.

For $x<y$ with $x,y\in D$, pointwise unbiasedness and the simultaneous threshold representation give
\begin{align}
y-x
&=\E_U[q_U(y)-q_U(x)]\\
&=\E_U\left[d_U\1\{x<r_U\le y\}\right].
\label{eq:weighted-threshold-measure}
\end{align}
Define the positive Borel measure
\[
\nu(A)=\E_U[d_U\1\{r_U\in A\}].
\]
Equation \eqref{eq:weighted-threshold-measure} shows that $\nu((x,y])=y-x$ whenever $x,y\in D$.  For arbitrary $x<y$, first choose $x_n\in D$ with $x_n\downarrow x$ and use continuity from below for $(x_n,y']$ with fixed $y'\in D$; then choose $y_n\in D$ with $y_n\downarrow y$ and use continuity from above on a finite enclosing interval.  Hence $\nu((x,y])=y-x$ for all $x<y$, so $\nu$ is Lebesgue measure.

Substituting \eqref{eq:equality-levels} now shows that the marginal law of $r_U$ has density
\[
\frac{\sqrt{F(r)(1-F(r))}}{\cJ(P)}=h_P(r).
\]
The decoder levels in \eqref{eq:equality-levels} are exactly those of \eqref{eq:optimal-decoder}.  This proves essential uniqueness.
\end{proof}

\section{Proof of the Gaussian Minimax Theorem}
\label{app:gaussian}

By scale equivariance it suffices to prove the theorem for $\sigma=1$.  Let
\[
P_c^\star=\tfrac12\cN(-c,1)+\tfrac12\cN(c,1),
\qquad
F_c(z)=\tfrac12\Phi(z-c)+\tfrac12\Phi(z+c).
\]
This distribution is symmetric and has mean zero.

\subsection{Least-favorable-prior lower bound}

For any code,
\begin{align*}
\sup_{|a|\le c}\E_a\widehat X^2
&\ge\tfrac12\left(\E_c\widehat X^2+\E_{-c}\widehat X^2\right)\\
&=\E_{X\sim P_c^\star}\widehat X^2\\
&\ge\Lambda_c^2,
\end{align*}
where the last step is Theorem~\ref{thm:fixed-source}.  Thus $P_c^\star$ is a candidate least-favorable prior.

\subsection{Second-moment monotonicity of the mixture-centered code}

Consider the code \eqref{eq:gaussian-code}.  For $z\ge0$, put
\[
q(z)=1-F_c(z)\le\tfrac12.
\]
At parameter $a\in[0,c]$, the bit success probabilities at thresholds $z$ and $-z$ are
\[
p_+(a,z)=1-\Phi(z-a),
\qquad
p_-(a,z)=\Phi(z+a).
\]
The two squared numerators in the code contribute
\begin{align}
H_a(z)
={}&q(z)^2+p_+(a,z)(1-2q(z))\\
&+(1-q(z))^2+p_-(a,z)(2q(z)-1)\\
={}&q(z)^2+(1-q(z))^2\\
&-(1-2q(z))\bigl(\Phi(z-a)+\Phi(z+a)-1\bigr).
\label{eq:paired-risk-app}
\end{align}
Differentiating gives
\begin{equation}
\partial_a H_a(z)
=(1-2q(z))\bigl(\phi(z-a)-\phi(z+a)\bigr)\ge0,
\label{eq:paired-risk-derivative}
\end{equation}
because $|z-a|\le z+a$ and $\phi$ is even and decreases with absolute value.  Since $\rho_c$ is symmetric,
\[
\Risk(a)=\int_0^\infty\frac{H_a(z)}{\rho_c(z)}\,\dd z
\]
is nondecreasing on $[0,c]$.  Reflection symmetry gives $\Risk(a)=\Risk(-a)$.  The two endpoint second-moment risks are equal and their average is the fixed-source risk under $P_c^\star$, namely $\Lambda_c^2$.  Hence
\[
\sup_{|a|\le c}\Risk(a)=\Lambda_c^2.
\]
Together with the lower bound, this proves \eqref{eq:gaussian-minimax-value}.

If a minimax code attains $\Lambda_c^2$, then both endpoint second-moment risks are at most $\Lambda_c^2$, while their average is at least $\Lambda_c^2$.  Equality therefore holds in the fixed-source theorem for $P_c^\star$, and Theorem~\ref{thm:uniqueness} gives the asserted essential uniqueness.  Restoring scale multiplies decoder outputs by $\sigma$ and second moments by $\sigma^2$.

\section{Heavy-Tail Necessity and Robustification}
\label{app:tail}

\subsection{Moment phase transition}

Fix $c\ge0$ and write
\[
F(z)=F_c(z),\qquad q(z)=1-F_c(z),\qquad
\rho(z)=\rho_c(z)=\frac{\sqrt{F(z)q(z)}}{\Lambda_c}.
\]
Under $X\sim\cN(a,1)$, let $p_a(z)=\Pp(X\ge z)=\barPhi(z-a)$.  The $p$th absolute moment of the standardized reconstruction is
\begin{equation}
\E_a|W_{a,c}|^p
=\int_{\R}
\left[p_a(z)F(z)^p+(1-p_a(z))q(z)^p\right]
\rho(z)^{1-p}\,\dd z.
\label{eq:p-moment-integral}
\end{equation}
On compact sets this integrand is finite.  We analyze the two tails.

For $z\to+\infty$,
\begin{equation}
q(z)=\tfrac12\barPhi(z-c)+\tfrac12\barPhi(z+c)
\asymp \barPhi(z-c),
\qquad F(z)\to1.
\label{eq:q-asymptotic}
\end{equation}
The second term in \eqref{eq:p-moment-integral} is integrable for every fixed $p>0$.  The first term is, up to multiplicative constants,
\begin{equation}
\barPhi(z-a)\,\barPhi(z-c)^{-(p-1)/2}.
\label{eq:moment-tail-comparison}
\end{equation}
Mills' ratio gives
\begin{align}
\log\left(
\barPhi(z-a)\barPhi(z-c)^{-(p-1)/2}
\right)
={}&\frac{p-3}{4}z^2\\
&+\left(a-\frac{p-1}{2}c\right)z+O(\log z).
\label{eq:moment-tail-exponent}
\end{align}
Therefore the positive tail is integrable for every $a\le c$ when $p<3$, is integrable for $p=3$ precisely when $a<c$, and diverges for every $a$ when $p>3$.  The negative tail is obtained by replacing $a$ by $-a$.  This proves the moment classification in Theorem~\ref{thm:tail-phase}.

\subsection{Endpoint survival asymptotics}

At $a=c$, a large positive reconstruction occurs when a large positive threshold is crossed.  Define
\[
w(z)=\frac{F(z)}{\rho(z)}
=\Lambda_c\sqrt{\frac{F(z)}{q(z)}}.
\]
This is strictly increasing for all sufficiently large $z$.  Let $z_t$ solve $w(z_t)=t$.  Then
\begin{equation}
q(z_t)\asymp t^{-2},
\qquad
z_t-c\asymp\sqrt{\log t}.
\label{eq:zt-asymptotic}
\end{equation}
For large $t$,
\begin{align}
\Pp_c(W_{c,c}>t)
&=\int_{z_t}^{\infty}\rho(z)\barPhi(z-c)\,\dd z\\
&\asymp\int_{z_t}^{\infty}q(z)^{3/2}\,\dd z.
\label{eq:tail-integral-q}
\end{align}
Mills' ratio and one integration by parts imply
\begin{equation}
\int_{z}^{\infty}q(u)^{3/2}\,\dd u
\asymp\frac{q(z)^{3/2}}{z-c}.
\label{eq:q-tail-integral}
\end{equation}
Substituting \eqref{eq:zt-asymptotic} proves
\[
\Pp_c(W_{c,c}>t)
=\Theta\!\left(t^{-3}(\log t)^{-1/2}\right).
\]
The negative endpoint follows by reflection.  Essential uniqueness from Theorem~\ref{thm:gaussian-minimax} transfers this obstruction to every attaining minimax code.

\subsection{Cauchy-mixture robustification}
\label{app:robustification}

Let
\[
r(z)=\frac1{\pi(1+z^2)},
\qquad
\rho_{c,\eta}(z)=(1-\eta)\rho_c(z)+\eta r(z).
\]
Draw $Z\sim\rho_{c,\eta}$ and decode
\[
W_\eta(x,Z)
=\frac{\1\{x\ge Z\}-(1-F_c(Z))}{\rho_{c,\eta}(Z)}.
\]
For every fixed $x$,
\begin{align*}
\E_ZW_\eta(x,Z)
&=\int_{\R}\left(\1\{x\ge z\}-(1-F_c(z))\right)\,\dd z\\
&=x,
\end{align*}
because $P_c^\star$ has mean zero.  Under the endpoint mixture,
\begin{align}
\E_{P_c^\star}W_\eta^2
&=\int_{\R}\frac{F_c(z)(1-F_c(z))}{\rho_{c,\eta}(z)}\,\dd z\\
&\le\frac1{1-\eta}\int_{\R}\frac{F_c(z)(1-F_c(z))}{\rho_c(z)}\,\dd z\\
&=\frac{\Lambda_c^2}{1-\eta}.
\label{eq:robust-mixture-risk}
\end{align}
The paired-risk derivative \eqref{eq:paired-risk-derivative} only uses symmetry and a positive common denominator for $z$ and $-z$.  Since $\rho_{c,\eta}$ is symmetric, the worst-case risk remains at the endpoints.  This proves \eqref{eq:robust-risk}.

For moments, $\rho_{c,\eta}(z)\ge\eta/[\pi(1+z^2)]$.  Hence the magnitude of either decoder branch is at most $C_\eta(1+z^2)$.  A large branch at $z\to+\infty$ is selected with Gaussian probability $\barPhi(z-a)$, and the reflected statement holds at $-\infty$.  Consequently every polynomial moment is integrable uniformly over $|a|\le c$.  This proves Corollary~\ref{cor:robustified}.

\subsection{Sensitivity to scale misspecification}

The pointwise-unbiased identity does not require the design scale to equal the true noise scale.  If the encoder uses $\widetilde\sigma>0$, thresholds $\widetilde\sigma Z$, and decoder multiplier $\widetilde\sigma$, then $\E[\widehat x\mid x]=x$ for every $x$.  The attaining code's second moment can nevertheless become infinite.  Put
\[
\tau=\frac{\sigma}{\widetilde\sigma},
\qquad
a=\frac{\mu}{\widetilde\sigma}.
\]
The positive-tail exponent is
\[
-\frac{(z-a)^2}{2\tau^2}+\frac{(z-c)^2}{4}.
\]
Thus the second moment is finite if $\tau<\sqrt2$, is finite at $\tau=\sqrt2$ precisely when $|a|<c$, and is infinite when $\tau>\sqrt2$.  The Cauchy-mixture code has all positive-order absolute moments for every finite $(\mu,\sigma,\widetilde\sigma)$, providing a robust alternative when the noise scale is uncertain.

\section{Finite Public Randomness and Finite Precision}
\label{app:finite}

\subsection{Why exact full-line unbiasedness needs infinitely many public-seed states}

\begin{proof}[Proof of Proposition~\ref{prop:finite-impossible}]
Let the public seed take values in a finite set $\mathcal U$.  After replacing decoder randomness by conditional means, the decoder has finitely many values
\[
\{a_0(u),a_1(u):u\in\mathcal U\}.
\]
For every input $x$, the conditional mean reconstruction lies in the convex hull of this finite set, hence in a bounded interval.  It cannot equal $x$ for all $x\in\R$.
\end{proof}

The argument also shows that a uniformly bounded decoder cannot be exactly pointwise unbiased on the full real line, irrespective of the cardinality of the public seed.

\subsection{Probability rounding}

Recall the midpoint grid
\[
r_i=-L+(i-\tfrac12)\Delta,
\qquad \Delta=\frac{2L}{K},
\qquad i=1,\ldots,K,
\]
and let
\[
q_i=1-F_c(r_i),\quad v_i=q_i(1-q_i),\quad
S_{K,L}=\Delta\sum_{i=1}^K\sqrt{v_i},\quad
\pi_i^\star=\frac{\Delta\sqrt{v_i}}{S_{K,L}}.
\]

\begin{lemma}[Finite-state probability rounding]
\label{lem:prob-rounding}
For every integer $N>K$, there are integers $n_i\ge1$, $\sum_i n_i=N$, such that
\[
\frac{n_i}{N}\ge\left(1-\frac KN\right)\pi_i^\star
\qquad\text{for every }i.
\]
\end{lemma}

\begin{proof}
Set
\[
n_i^0=\left\lfloor (N-K)\pi_i^\star\right\rfloor+1.
\]
Then $n_i^0\ge1$, $n_i^0\ge(N-K)\pi_i^\star$, and
\[
\sum_i n_i^0\le (N-K)\sum_i\pi_i^\star+K=N.
\]
Distribute the remaining $N-\sum_i n_i^0$ states arbitrarily.  Dividing by $N$ proves the claim.
\end{proof}

A uniform shared-randomness string of $R$ bits, $N=2^R$, can therefore be mapped to the $K$ thresholds with probabilities $\pi_i=n_i/N$.  Duplicate public-seed states assigned to the same threshold use the same decoder table entry.  A direct representation of the counts uses $O(KR)$ bits, while an implementation may instead store a cumulative seed-to-threshold table or generate the mapping offline.

\subsection{Bias, second moment, and output magnitude}

For a deterministic standardized input $x$, the conditional expectation of \eqref{eq:finite-code} is
\begin{align}
\E[\widehat X_{K,L,R}\mid x]
&=\sigma\Delta\sum_{i=1}^K\bigl(\1\{x\ge r_i\}-q_i\bigr).
\label{eq:finite-conditional-mean}
\end{align}
Because the grid and $F_c$ are symmetric, $q_i+q_{K+1-i}=1$, so
\[
\Delta\sum_{i=1}^Kq_i=L.
\]
Thus the bracketed expression is the midpoint staircase approximation to $\clip(x,-L,L)$, and
\begin{equation}
\left|\frac1\sigma\E[\widehat X_{K,L,R}\mid x]
-\clip(x,-L,L)\right|
\le\frac\Delta2=\frac LK.
\label{eq:midpoint-staircase}
\end{equation}

Let $Y\sim\cN(a,1)$ with $|a|\le c$ and $L=c+t$.  The clipping bias satisfies
\begin{align}
\left|\E\clip(Y,-L,L)-a\right|
&\le\E(Y-L)_++\E(-L-Y)_+\\
&\le2\left(\phi(t)-t\barPhi(t)\right).
\label{eq:clipping-bias}
\end{align}
Combining \eqref{eq:midpoint-staircase} and \eqref{eq:clipping-bias} proves \eqref{eq:finite-bias}.

We next control the second moment.  For the ideal probabilities $\pi_i^\star$, define
\[
R_a^\star
=\sum_{i=1}^K\frac{\Delta^2}{\pi_i^\star}
\left[p_{a,i}(1-q_i)^2+(1-p_{a,i})q_i^2\right],
\]
where $p_{a,i}=\barPhi(r_i-a)$.  Pairing thresholds $z$ and $-z$ gives exactly the numerator $H_a(z)$ in \eqref{eq:paired-risk-app}.  Since $\pi_i^\star=\pi_{K+1-i}^\star$, \eqref{eq:paired-risk-derivative} implies that $R_a^\star$ is nondecreasing in $|a|$.  At the endpoints, reflection symmetry makes the two risks equal, and their average under $P_c^\star$ is
\begin{align}
R_c^\star
&=\sum_{i=1}^K\frac{\Delta^2v_i}{\pi_i^\star}
=\left(\Delta\sum_{i=1}^K\sqrt{v_i}\right)^2
=S_{K,L}^2.
\label{eq:discrete-ideal-risk}
\end{align}
Hence $\sup_{|a|\le c}R_a^\star=S_{K,L}^2$.

For the rounded probabilities, Lemma~\ref{lem:prob-rounding} gives termwise
\[
\frac1{\pi_i}\le\frac1{1-K/N}\frac1{\pi_i^\star}.
\]
Therefore
\[
\sup_{|a|\le c}R_a
\le\frac{S_{K,L}^2}{1-K/N},
\]
which proves \eqref{eq:finite-risk} after restoring scale.

Finally $|M-q_i|\le1$ and $\pi_i\ge1/N$, so
\[
|\widehat X_{K,L,R}|
\le\sigma\Delta N.
\]
This proves \eqref{eq:finite-amplitude}.

If $N\ge K(1+\eta)/\eta$, then $K/N\le\eta/(1+\eta)$ and $(1-K/N)^{-1}\le1+\eta$.  Taking
\[
R=\left\lceil\log_2\frac{K(1+\eta)}{\eta}\right\rceil
\]
therefore suffices.

\subsection{Convergence and stored precision}

Let
\[
\psi_c(z)=\sqrt{F_c(z)(1-F_c(z))},
\qquad
D_{c,L}=\sup_{|z|\le L}|\psi_c'(z)|.
\]
The midpoint rule gives
\begin{equation}
\left|S_{K,L}-\int_{-L}^{L}\psi_c(z)\,\dd z\right|
\le\frac{D_{c,L}L^2}{K}.
\label{eq:midpoint-error}
\end{equation}
The Gaussian tail bound $\barPhi(u)\le\phi(u)/u$ gives, for $t=L-c\ge1$,
\begin{equation}
0\le\Lambda_c-\int_{-L}^{L}\psi_c(z)\,\dd z
\le4(2\pi)^{-1/4}t^{-3/2}e^{-t^2/4}.
\label{eq:lambda-tail-truncation}
\end{equation}
Thus the finite-code constant converges explicitly to $\Lambda_c$.

The shared threshold grid can be chosen dyadic, so each standardized threshold has $O(\log K+\log L)$ bits of storage.  Suppose the two reconstruction levels associated with every threshold are rounded so that the implemented output $\widetilde X$ satisfies
\[
|\widetilde X-\widehat X_{K,L,R}|\le\tau\sigma
\quad\text{almost surely}.
\]
Then
\begin{align*}
\left|\E\widetilde X-\mu\right|
&\le\left|\E\widehat X_{K,L,R}-\mu\right|+\tau\sigma,\\
\sqrt{\E\widetilde X^2}
&\le\sqrt{\E\widehat X_{K,L,R}^2}+\tau\sigma
\end{align*}
by Jensen and Minkowski.  Since the unrounded levels are bounded by $\sigma\Delta N$, storing each level to absolute precision $\tau\sigma$ requires
$O(\log(\Delta N/\tau))$ magnitude bits.  This completes the proof of Theorem~\ref{thm:finite-coin}.

\section{Vector Compression and Fixed-Bit Accounting}
\label{app:vector}

Let $s=\min\{B,d\}$ and draw a uniform subset $S\subseteq[d]$ of cardinality $s$ using fresh public randomness.  For $j\in S$, let $Q_j^{\rm sc}(g_j)$ be an independent scalar reconstruction from Theorem~\ref{thm:gaussian-minimax}; for $j\notin S$, return zero.  Define
\[
[\cQ_B(g)]_j=\frac ds\1\{j\in S\}Q_j^{\rm sc}(g_j).
\]
For every deterministic vector $g$,
\[
\E[\cQ_B(g)_j\mid g]
=\frac sd\frac ds\E[Q_j^{\rm sc}(g_j)\mid g_j]
=g_j.
\]
If $g\sim\cN(\mu,\sigma^2I_d)$ and $\|\mu\|_\infty\le c\sigma$, then
\begin{align*}
\E\|\cQ_B(g)\|^2
&=\sum_{j=1}^d\left(\frac ds\right)^2\Pp(j\in S)
\E[Q_j^{\rm sc}(g_j)^2]\\
&\le\frac ds\cdot d\sigma^2\Lambda_c^2
=\Lambda_c^2\frac{\sigma^2d^2}{s}.
\end{align*}
Exactly one bit is sent for each selected coordinate.  The subset and ordering are public.  If $B>d$, the remaining positions are deterministic padding.

For adaptive optimization, round-$t$ thresholds and subsets must be fresh after $x_t$ is selected.  Equivalently, split the public seed into independent blocks $U_1,\ldots,U_T$ and require $x_t$ to depend only on the previous blocks and transcript.  This is a permitted specialization of the public-coin oracle model.

\section{Expected Quadratic Optimization Guarantee}
\label{app:quadratic}

At $x_t=0$ in \eqref{eq:quadratic-family},
\[
g_t=-\theta+\xi_t\sim\cN(-\theta,\sigma^2I_d),
\qquad
\|\theta\|_\infty\le\delta\le\sigma/2.
\]
Let $\widetilde g_t=\cQ_B(g_t)$ using $c=1/2$.  Then $\E\widetilde g_t=-\theta$, and independent oracle noise and fresh public randomness give
\begin{align*}
\E\left\|-\frac1T\sum_{t=1}^T\widetilde g_t-\theta\right\|^2
&=\frac1{T^2}\sum_{t=1}^T\E\|\widetilde g_t+\theta\|^2\\
&=\frac1T\Var(\widetilde g_1)\\
&\le\frac1T\E\|\widetilde g_1\|^2\\
&\le\frac{\Lambda_{1/2}^2\sigma^2d^2}{T\min\{B,d\}}.
\end{align*}
Projection onto the cube containing $\theta$ is nonexpansive.  Since
\[
f_\theta(x)-f_\theta^\star=\frac12\|x-\theta\|^2,
\]
this proves \eqref{eq:quadratic-upper}.

For comparison, Theorem~3 of \citet{kim2026bitconstrained} states that at the scale \eqref{eq:hard-scale}, any protocol with expected gap at most $\eps^2$ must satisfy
\[
T\ge\frac{\sigma^2d^2}{4\eps^2\min\{2\ln(2)B,d\}}.
\]
The upper and lower bounds therefore have identical dependence on $(\sigma,d,B,\eps)$.

\subsection{Finite-randomness implementation}

We prove Corollary~\ref{cor:finite-quadratic}.  Let $s=\min\{B,d\}$ and assume $n:=Ts/d$ is an integer.  A deterministic cyclic schedule assigns exactly $n$ scalar observations to every coordinate, so no shared randomness is needed to specify coordinate indices.  Each active coordinate uses an independent $R$-bit shared seed for the code of Theorem~\ref{thm:finite-coin}; when $B>d$, the remaining message positions are deterministic padding.

For coordinate $j$, let $Y_{j,1},\ldots,Y_{j,n}$ be the scalar reconstructions of Gaussian observations with mean $-\theta_j$.  By Theorem~\ref{thm:finite-coin} and \eqref{eq:finite-constants},
\[
\left|\E Y_{j,\ell}+\theta_j\right|\le\sigma\beta_{K,L},
\qquad
\E Y_{j,\ell}^2\le\sigma^2V_{K,L,R}.
\]
Set $\widetilde\theta_j=-n^{-1}\sum_{\ell=1}^nY_{j,\ell}$ and project $\widetilde\theta$ onto the parameter cube.  Independence across repeated observations gives
\begin{align*}
\E(\widetilde\theta_j-\theta_j)^2
&=\frac{\Var(Y_{j,1})}{n}+\left(\E Y_{j,1}+\theta_j\right)^2\\
&\le\frac{\sigma^2V_{K,L,R}}{n}+\sigma^2\beta_{K,L}^2.
\end{align*}
Summing over coordinates, using projection nonexpansiveness, and multiplying by $1/2$ proves \eqref{eq:finite-quadratic}.  If $d\nmid Ts$, fewer than $d/\gcd(d,s)$ additional rounds make the divisibility condition hold; the symbol $T$ in Corollary~\ref{cor:finite-quadratic} denotes the resulting total number of rounds.  The stored-level rounding statement follows from the bias and root-second-moment perturbation bounds in Theorem~\ref{thm:finite-coin}.

\section{High-Probability Quadratic Estimation from Finite Variance}
\label{app:high-probability}

We give an elementary median-of-means proof of Corollary~\ref{cor:quadratic-hp}.  The protocol uses a deterministic cyclic schedule assigning $s=\min\{B,d\}$ distinct coordinates per round.  Thus coordinate $j$ receives
\[
n_j\in\left\{\left\lfloor\frac{Ts}{d}\right\rfloor,
\left\lceil\frac{Ts}{d}\right\rceil\right\}
\]
independent scalar reconstructions.  No Horvitz--Thompson factor is needed in this offline coordinatewise estimator.  Each scalar reconstruction $Y_{j,\ell}$ satisfies
\[
\E Y_{j,\ell}=-\theta_j,
\qquad
\Var(Y_{j,\ell})\le\E Y_{j,\ell}^2
\le\Lambda_{1/2}^2\sigma^2.
\]

We use the following standard finite-variance lemma.

\begin{lemma}[Scalar median of means]
\label{lem:mom}
Let $Y_1,\ldots,Y_n$ be independent with common mean $\nu$ and variance at most $v$.  Let $k$ be an odd integer with $n\ge2k$.  Split the first $k\lfloor n/k\rfloor$ samples into $k$ equal blocks and take the median of the block means, denoted $\widehat\nu_{\rm MOM}$.  Then
\[
\Pp\!\left(
|\widehat\nu_{\rm MOM}-\nu|^2>\frac{8vk}{n}
\right)
\le e^{-k/8}.
\]
\end{lemma}

\begin{proof}
Let $m=\lfloor n/k\rfloor\ge n/(2k)$.  Each block mean has variance at most $v/m$, so Chebyshev gives
\[
\Pp\left(|\overline Y_b-\nu|>2\sqrt{v/m}\right)\le\frac14.
\]
If the median is outside this interval, at least half the blocks are bad.  Hoeffding's inequality bounds this event by $e^{-k/8}$.  Finally $4v/m\le8vk/n$.
\end{proof}

Choose $k$ to be the smallest odd integer at least $8\log(2d/\delta_0)$.  When $Ts\ge4dk$, each $n_j\ge2k$.  Applying Lemma~\ref{lem:mom} to every coordinate and taking a union bound gives, with probability at least $1-\delta_0$,
\begin{align*}
\|\widehat\theta-\theta\|^2
&\le\sum_{j=1}^d\frac{8\Lambda_{1/2}^2\sigma^2k}{n_j}\\
&\le16\Lambda_{1/2}^2\sigma^2
\frac{d^2k}{Ts}.
\end{align*}
Projection onto the cube can only decrease this error, and multiplication by $1/2$ gives \eqref{eq:quadratic-hp} with a universal constant.

\section{Localized Strongly-Convex Extension}
\label{app:general-sgd}

The quadratic result does not require iterative gradient descent.  Nevertheless, the compressor can be inserted into projected SGD without assuming bounded stochastic-gradient realizations.

\begin{theorem}[Projected SGD with localized Gaussian gradient means]
\label{thm:localized-sgd}
Let $\mathcal X\subseteq\R^d$ be closed and convex.  Suppose $f$ is differentiable, $\mu$-strongly convex, and $L$-smooth on $\mathcal X$, and its unconstrained minimizer $x^\star$ lies in $\mathcal X$.  Let
\[
g_t=\nabla f(x_t)+\xi_t,
\qquad
\xi_t\stackrel{\rm iid}{\sim}\cN(0,\sigma^2I_d),
\]
and assume only
\[
\sup_{x\in\mathcal X}\|\nabla f(x)\|_\infty\le c\sigma.
\]
Run
\[
x_{t+1}=\Pi_{\mathcal X}\bigl(x_t-\eta_t\cQ_B(g_t)\bigr),
\qquad
\eta_t=\frac1{\mu(t+1)},
\]
using fresh public randomness.  With
\[
M_c=\Lambda_c^2\frac{\sigma^2d^2}{\min\{B,d\}},
\]
we have
\begin{align}
\E\|x_{T+1}-x^\star\|^2
&\le\frac1{T+1}
\max\left\{\|x_1-x^\star\|^2,\frac{M_c}{\mu^2}\right\},
\label{eq:localized-distance}\\
\E[f(x_{T+1})-f(x^\star)]
&\le\frac{L}{2(T+1)}
\max\left\{\|x_1-x^\star\|^2,\frac{M_c}{\mu^2}\right\}.
\label{eq:localized-function}
\end{align}
\end{theorem}

\begin{proof}
Let $a_t=\E\|x_t-x^\star\|^2$.  Projection nonexpansiveness, compressor unbiasedness, and strong monotonicity give
\begin{align*}
a_{t+1}
&\le a_t-2\eta_t\E\langle x_t-x^\star,\nabla f(x_t)\rangle
+\eta_t^2\E\|\cQ_B(g_t)\|^2\\
&\le(1-2\mu\eta_t)a_t+M_c\eta_t^2\\
&=\frac{t-1}{t+1}a_t+\frac{M_c}{\mu^2(t+1)^2}.
\end{align*}
Set $C=\max\{a_1,M_c/\mu^2\}$.  A direct induction gives $a_t\le C/t$.  Since $\nabla f(x^\star)=0$ and $f$ is $L$-smooth,
\[
f(x)-f(x^\star)\le\frac L2\|x-x^\star\|^2.
\]
This proves both claims.
\end{proof}

The assumption controls only the Gaussian mean.  It permits arbitrarily large stochastic-gradient realizations and is therefore distinct from an almost-sure dynamic-range bound.  Removing all mean localization additionally requires locating an arbitrarily translated optimum.

\section{Numerical Evaluation and Figure Generation}
\label{app:numerics}

For $c\ge0$,
\[
F_c(z)=\tfrac12\Phi(z-c)+\tfrac12\Phi(z+c),
\qquad
\Lambda_c=2\int_0^\infty\sqrt{F_c(z)(1-F_c(z))}\,\dd z.
\]
The script \texttt{scripts/compute\_constants.py} uses arbitrary-precision quadrature.  It returns
\[
\Lambda_{1/2}=1.79387486539778157772392605329\ldots,
\]
\[
\Lambda_{1/2}^2=3.21798703270590897323621989582\ldots.
\]
The script \texttt{scripts/generate\_figures.py} evaluates the one-dimensional risk integrals, the parametric survival curve, and the finite-grid constants used in Figure~\ref{fig:diagnostics}.  The clean code archive also contains \texttt{reproduce\_all.sh}, reusable scalar-code modules under \texttt{src/}, and tests for the constants, probability rounding, and finite-grid calculations.  Repository and supplementary-archive details are given in the Reproducibility Statement.  Numerical calculations illustrate the theory but are not used in any proof.

\section{Resource Accounting and Scope Summary}
\label{app:resources}

\begin{table}[H]
\centering
\caption{Summary of communication resources, implementation guarantees, and scope.}
\begin{tabular}{>{\raggedright\arraybackslash}p{0.29\textwidth}>{\raggedright\arraybackslash}p{0.62\textwidth}}
\toprule
Aspect & Guarantee or limitation\\
\midrule
Gaussian observations & Comparisons are defined for every real realization; there is no overflow message or clipping event.\\
Scale information & The only scale is the model/design parameter $\sigma$; no sample-dependent real is transmitted.\\
Coordinate allocation & Selected coordinates and their order are public.\\
Message length & Every optimization round sends exactly $B$ bits, including deterministic padding for $B>d$.\\
Causality and fresh randomness & Round-$t$ public coins are first used after $x_t$ is chosen.\\
Exact public randomness & Infinite public support is necessary for exact full-line unbiasedness with finite decoder means (Proposition~\ref{prop:finite-impossible}).\\
Finite implementation & Theorem~\ref{thm:finite-coin} quantifies shared random bits, bias, second moment, output magnitude, and stored-level precision; Corollary~\ref{cor:finite-quadratic} gives the end-to-end optimization bound.\\
Output tails & Attaining the minimax value forces the order-three phase transition; Corollary~\ref{cor:robustified} restores all positive-order absolute moments with arbitrarily small second-moment slack.\\
High-probability guarantee & Median-of-means uses only finite variance and gives Corollary~\ref{cor:quadratic-hp}.\\
Uniform-source precedent & Explicitly cited and recovered as Corollary~\ref{cor:uniform}; novelty is not claimed for that special case.\\
Atoms and bounded support & The fixed-source infimum uses the distributional transform; interval-unbiased codes attain it, while full-line codes approach it by a full-support perturbation.\\
Scalar and vector guarantees & Exact constants are scalar.  The induced optimization scheme is rate-optimal on Kim's continuous quadratic hard family; no joint-vector constant optimality is claimed.\\
Optimization scope & Exact lower-bound matching is for Kim's continuous quadratic hard family; broader SGD needs mean localization.\\
\bottomrule
\end{tabular}
\end{table}